\documentclass[10pt,letterpaper]{article}

\usepackage[utf8]{inputenc}

\usepackage{cite}
\usepackage{amssymb}
\usepackage{lipsum}

\usepackage{amsthm}

\usepackage{amssymb,amsmath,amsthm}

\newtheorem{theorem}{Theorem}

\newtheorem{lemma}[theorem]{Lemma}

\usepackage{mathtools} 

\usepackage{nameref,hyperref}

\usepackage{microtype}
\DisableLigatures[f]{encoding = *, family = * }

\usepackage{changepage}

\usepackage{amsfonts}
\usepackage{amsmath}   
\usepackage{url}       
\usepackage{subfigure}
\usepackage{psfrag} 
\usepackage{graphicx}
\makeatletter
\renewcommand{\@biblabel}[1]{\quad#1.}
\makeatother

\usepackage{lastpage,fancyhdr,graphicx}
\usepackage{epstopdf}
\usepackage{color}

\definecolor{Gray}{gray}{.25}

\usepackage{graphicx}

\usepackage{sidecap}

\usepackage{wrapfig}
\usepackage[pscoord]{eso-pic}
\usepackage[fulladjust]{marginnote}
\reversemarginpar

\begin{document}
\vspace*{0.35in}
\newtheoremstyle{mythm}{3pt}{3pt}{}{}{\itshape}{:}{.5em}{}
\theoremstyle{mythm}
\begin{flushleft}
{\Large
\textbf\newline{Glucose values prediction five years ahead with a new framework of missing responses in reproducing kernel Hilbert spaces,  and the use of continuous glucose monitoring technology}
}
\newline
\\
Marcos Matabuena\textsuperscript{1,*},
Paulo Félix\textsuperscript{1},
Carlos Meijide-Garcia\textsuperscript{2},
Francisco Gude\textsuperscript{3}
\\
\bigskip
\bf{1} CiTIUS (Centro Singular de Investigaci\'{o}n en Tecnolox\'{i}as Intelixentes), Universidade de Santiago de Compostela, Spain
\\
\bf{2} Universidade de Santiago de Compostela, Spain
\\
\bf{3} Unidade de Epidemiolox\'{i}a Cl\'{i}nica, Hospital Cl\'{i}nico Universitario de  Santiago de Compostela, Spain
\\
\bigskip
*\url{marcos.matabuena@usc.es}

\end{flushleft}

\section*{Abstract}
%

 AEGIS study possesses unique information on longitudinal changes in circulating glucose through continuous glucose monitoring technology (CGM). However, as usual in longitudinal medical studies, there is a significant amount of missing data in the outcome variables.  For example, 40 percent of glycosylated hemoglobin (A1C) biomarker data are missing five years ahead. With the purpose to reduce the impact of this issue, this article proposes a new data analysis framework based on learning in reproducing kernel Hilbert spaces (RKHS) with missing responses that allows to capture non-linear relations between variable studies in different supervised modeling tasks. First, we extend the Hilbert-Schmidt dependence measure to test statistical independence in this context introducing a new bootstrap procedure, for which we prove consistency. Next, we adapt or use existing models of variable selection, regression, and conformal inference to obtain new clinical findings about glucose changes five years ahead with the AEGIS data. The most relevant findings are summarized below: i) We identify new factors associated with long-term glucose evolution; ii) We show the clinical sensibility of CGM data to detect changes in glucose metabolism; iii)  We can improve clinical interventions based on our algorithms' expected glucose changes according to patients' baseline characteristics.

\section*{Motivation and outline contributions}

With advances in digital patient monitoring and personalized medicine, a new clinical paradigm based on optimizing medical decisions according to data-driven approaches, is emerging. \textit{Diabetes mellitus} is an essential reference point to the application of these techniques. It is estimated that approximately 50\% of diabetes patients have not been diagnosed yet. Furthermore, adherence and effectiveness of treatments are poor across many patient groups; and disease prevalence is increasing with contemporary lifestyles. In this sense, predictive models that forecast and identify risk factors associated with the evolution of glycemic profiles in the short and long term is vital for identifying patients at risk of disease development, improving early diagnosis, and prescribing optimal dynamic treatments. This paper’s primary goal is to study the relationship between the AEGIS study patients’ baseline characteristics and the primary biomarker of diabetes diagnosis and control- glycosylated hemoglobin (A1C)- five years ahead. In addition, we introduce information about continuous glucose monitoring (CGM) into the models to capture individual glucose homeostasis fluctuations at a high-resolution level. As five-year A1C data registries are missing for approximately 40\% of patients, we propose a new data-analysis framework based on RKHS learning with missing responses as a methodological contribution. This machine learning (ML) paradigm allows to detect complex non-linear relations between study variables and analyzes simultaneous data of different nature from several information sources such as CGM. In particular, we address the statistical independence testing problem via a new Hilbert-Schmidt criterium designed explicitly for this context, and we do several adaptions of existing model-free methods of variable selection,  regression models, and conformal inference algorithms. Using these models, we achieve new clinical findings: i) We identify some diabetes biomarkers associated with glucose variations in the standard clinical routine, both marginally and from a multivariate perspective, ii) We show the need to incorporate CGM technology to predict glucose changes in the long term, iii) We identify some risk patients’ phenotypes for which the model’s predictive capacity is moderate, and therefore more personalized follow-up is needed by them.



\section{Introduction}

Diabetes mellitus is one of the most critical public health problems being the ninth major cause of death of mortality worldwide \cite{zheng2018global, saeedi2020mortality}. At present, over $416$ and $47$ million patients have Type II and Type I diabetes respectively \cite{saeedi2019global} with estimated health costs of disease management that reach $760$ billion dollars \cite{williams2020global}. Moreover, several projections forecast a significant increase in prevalence in the following decades \cite{whiting2011idf, cho2018idf}. Considering the growth of this pandemic among the general population \cite{tabish2007diabetes, hu2015curbing,ginter2013type}, the need to pursue new health politics to enable early recognition of risk patients and improvement in the methodology of disease diagnosis in the standard clinical routine is noteworthy. Nowadays, around 50\% of patients with diabetes are undiagnosed \cite{saeedi2019global}, and the proliferation of sedentary lifestyles is more generalized between the population \cite{finkelstein2012obesity} being a significant causal factor \cite{ng2014global} of the progressive increase in the incidence of chronic diseases \cite{visscher2001public, zheng2018global}, or that the density curve of body mass index along different age-groups is taking higher and higher values in the over-height and obesity range \cite{flegal2012prevalence}. As a consequence, clinical complications and burden of health costs associated \cite{rubin1994health} with the impaired glycemic condition in the early stages of the disease in patients to whom no specific glycemic individual glucose homeostasis control interventions are performed  \cite{walker2010diet} will have a stronger impact on human condition. \cite{zheng2018global, dabelea2017association}
	
A new emerging clinical  paradigm  based on digital and precision medicine \cite{topol2010transforming, kosorok2019precision, schork2015personalized} can be a landmark to improve early diagnosis. In this context, clinical decisions, e.g., treatment prescription, can be optimized through the intensive use of statistical models and machine learning techniques \cite{kosorok2015adaptive,kosorok2019precision, zhao2011reinforcement, coronato2020reinforcement, cirillo2019big}, that exploit the rich source of information generated by patients’ monitoring \cite{li2017digital}.

In the particular case of diabetes  \cite{ellahham2020artificial, gunasekeran2020artificial, zou2018predicting}, the application of these models can be a valuable weapon to an improvement in the early identification of patients with a high risk of developing diabetes, the prediction of complications such as retinopathy, as well as dynamical prescriptions of optimal treatments\cite{tsiatis2019dynamic, zhang2012estimating, goldberg2012q} in patients with type I diabetes. In this case, the patient variables involved in the individual data-driven treatments routine \cite{doi:10.1080/01621459.2018.1537919} can include non-pharmacological interventions such as physical exercise or diet, insulin pumps, or other common drugs such as metformin \cite{walker2010diet, centers2011national}.

The current advances in device technology allow assessing patients’ glucose metabolism at a high-resolution level, capturing the individual differences in the glucose fluctuations at different time scales via continuous glucose monitoring (CGM) \cite{Zaccardi2018}. However, current standard clinical biomarkers of diabetes diagnosis and control such as glycosylated hemoglobin (A1C) or fasting plasma glucose (FPG) \cite{zhang2010a1c} capture only partially the temporal complexity of the glycaemic profiles, measuring summary characteristics as the mean glucose over the precedent 3-month (A1C) or a glucose value in a specific instant of time selected in the morning (FPG) \cite{selvin2007short}. CGM device has been used primarily in specific  risk managing situations concerning  patients with type I diabetes \cite{poolsup2013systematic}. However, their use is more popular in clinical routines because of decreased costs and technological improvement. As a consequence, more general applications  in both disease and healthy populations  are emerging even outside the field of diabetes \cite{Ludc201862}. Some relevant examples include the acquisition of new clinical knowledge in epidemiological studies, the screening of patients, the evaluation of the prognosis of patients with diabetes, and the optimization of the diet \cite{beck2019validation,zeevi2015personalized}. Nevertheless, from the methodological point of view, there are essential difficulties in exploiting the data generated by these devices in realistic environments where patients are monitored in free-living conditions, and glucose fluctuations are not temporally aligned between patients, being time series analysis unfeasible. In a recent work, \cite{matabuena2020glucodensities} introduced new functional profile of CGM data termed \textit{glucodensity} to overcome these limitations, (see Figure \ref{fig:graf1} for intuitive explanation about the new data analysis method). The results show that this alternative method can assess glucose homeostasis more accurately than the so-called time in range metrics, the current gold standard in CGM data handling \cite{battelino2019clinical,beck2019validation, wilmot2020time,nguyen2020review}. Moreover, this new proposal's application overcame some limitations of time in range metrics such as predefinition of target zones- which can depend on the study population- and the loss of information caused by a discretization in different intervals of the recorded information.

\begin{figure}[ht!]
	\centering
	\includegraphics[width=0.9\linewidth]{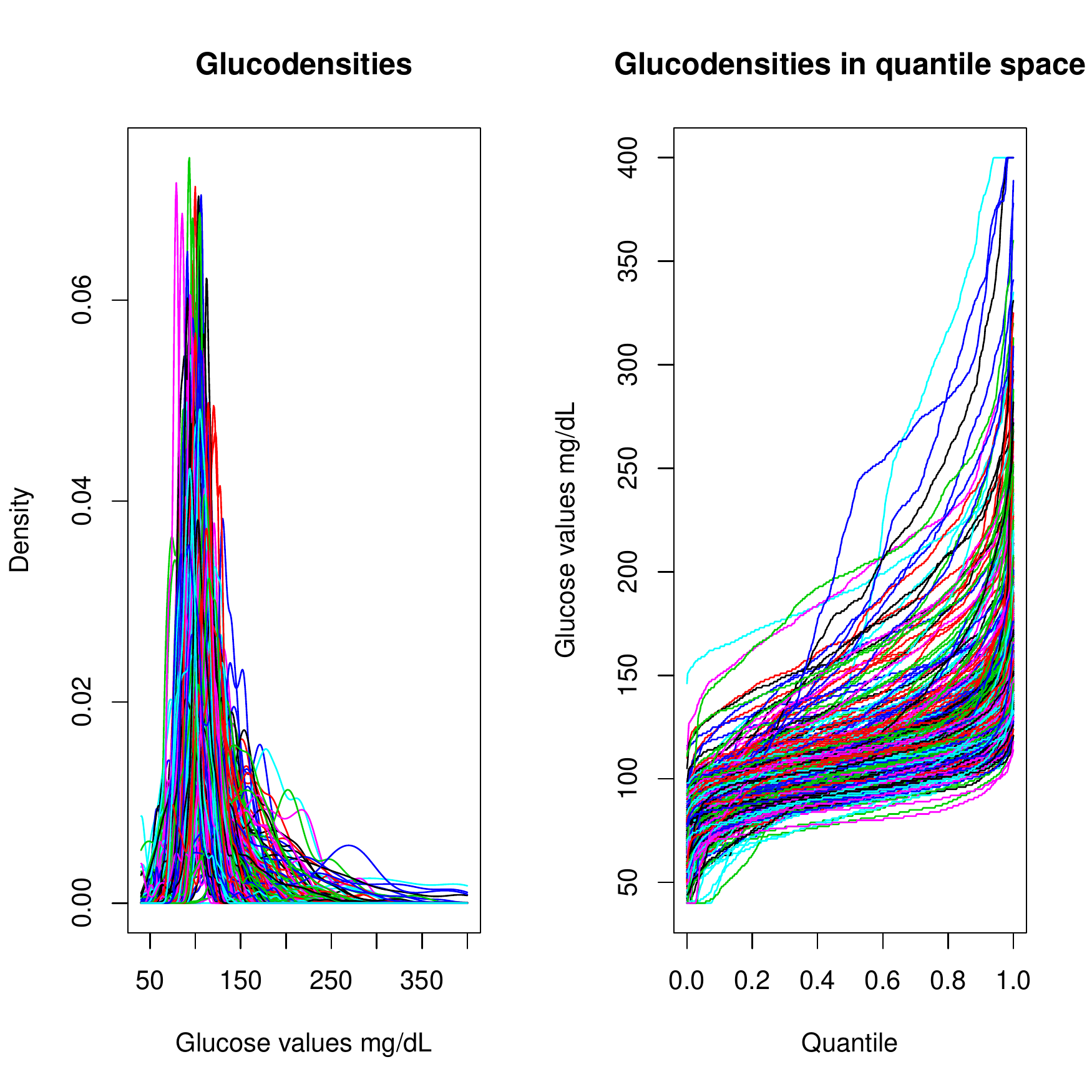}
	\caption{Glucodensities are estimated from a random sample of the AEGIS study with diabetic and normoglycemic patients. Our glucose representation estimates the proportion of time spent by a patient at each glucose concentration over a continuum. This represents a more sophisticated approach to assess glucose metabolism.  In the figure at the right, the representation of the glucodensities in the space of quantile functions is shown.}
	\label{fig:graf1}
\end{figure}

Several predictive models have been developed in the present body of scientific literature to control and diagnose diabetes \cite{edlitz2020prediction,zaitcev2020deep,lee2013prediction}. However, to the best of our knowledge, these models present several limitations as they do not incorporate the rich information about individual glucose homeostasis dynamics provided by CGM. Moreover, authors often fit the predictive models with data of observational nature without applying specific techniques to correct non-randomness in the sampling design, affecting generalization and inference in the predictive models. We believe that CGM technology can introduce new insight into the assessment of future glucose metabolism behavior. In addition, the use of high-quality data such as those obtained from a random sample of the general population is essential to obtain robust and reproducible conclusions about model performance. This study design, and not dealing with observational data, is the gold-standard practice to assess treatments' performance with safety in other domains such as clinical trials.

The AEGIS population-based study \cite{gude2017glycemic} is one of the most representative cohorts in the world that analyzes unique clinical characteristics about glucose dynamics over ten years, with a random sample of 1516 individuals from A Estrada (Galicia, Spain). At the beginning of this study, 581 participants were randomly selected because of the economical and logistical implications of wearing a CGM device for $3$-$7$ days. After a 5-year follow-up, a significant fraction of those individuals did not agree to perform a second monitoring, while some 5-year relevant outcomes such as A1C could not be measured in 40\% of the patients. This situation is commonplace in cohort studies, where the presence of missing data in different outcomes with a lack of follow-up of patients is familiar \cite{laird1988missing, tsiatis2007semiparametric}. Generally, in the literature of missing data when the outcome or response is missing, two different situations are considered \cite{tsiatis2007semiparametric,rubin1976inference,little2019statistical}. In the first case, MAR (\textit{Missing At Random}) assumption ensures that the mechanism of missing data is random and independent of covariates. In the second case, MNAR (\textit{Missing Not At Random}) hypothesis assumes that some covariates have a certain impact on the mechanism of missing data; for instance, in our example, older patients are less susceptible to perform a second CGM monitoring so that the probability of not observing a patient increase with age.

The aim of this paper is twofold. First, we will introduce new Machine Learning (ML) models to test statistical independence, perform variable selection, and predict and make inferences in a context where the response is missing in some individuals, something rather usual in the outcome variables AEGIS study. The proposed models are based on a Reproducible Kernel Hilbert Space (RKHS) learning paradigm, one of the most powerful machine learning strategies \cite{scholkopf2002learning, steinwart2008support}, that has been successfully applied in many real applications because of its ability to detect complex non-linear relations between study variables, see for example \cite{ivanciuc2007applications,noble2006support,salcedo2014support,deka2014support, muandet2016kernel, ghorbani2020neural}. Second, we apply the developed models to the AEGIS database, composed patients with and without diabetes. As a result, we identify markers associated with A1C-measured glucose values five years and predict future glucose values to acquire future changes in patient conditions.

\subsection{Predictive models of glucose evolution}

With the aim of stratification patients risk, different diabetes scores as the Finnish (FINDRISC) \cite{makrilakis2011validation} and the German (GDRS) \cite{muhlenbruch2018derivation} ones try to predict the probability of developing diabetes in ten years time with a logistic regression or the time to becoming a diabetic person with survival models such as Cox regression. The variables included in the models are easily obtained, such as age, sex, anthropometric measurements, or other clinical history information such as lifestyle, family history, and medication. More recently, Yochai Edlitz and Eran Segal \cite{edlitz2020prediction} propose several alternative ML models based on boosting gradient machine algorithms that involve different variables of greater or lesser complexity of measurement, such as the measures discussed above or laboratory biochemical biomarkers, or even genotyping variables. However, the UK Biobank sample is observational, and the authors do not use specific techniques to avoid the biases associated with the sampling mechanism, which limits generalization and reproducibility of the obtained results. Other articles present in the literature predict continuous biomarkers of diagnosis and control of diabetes such as FPG or A1C in both the non-diabetic and diabetic patients instead of predicting the event “development of diabetes” or another categorical variable \cite{zaitcev2020deep,lee2013prediction}.

We consider that the latter approach has considerable advantages respect the first mainly for two reasons: i) In terms of statistical inference and interpretation of the results, it is more precise  to predict a continuous variable rather than a categorized one, in which there is an evident loss of information; ii) from the clinical point of view, we lose the information of each individual’s glucose values with the categorization of the variables. Suppose we do not categorize the variables and use the biomarkers’ value as a continuous variable. In that case, we can analyze a target sample involving both patients without diabetes and patients with diabetes, as with the AEGIS study database.

In this sense, some authors have recently affirmed that the current geocentric definition of diabetes may be strict in various situations \cite{VAS2017848}. Using CGM technology with multitudes of measurements of an individual’s glucose metabolism may lead to the establishment of more personalized diagnostic thresholds \cite{Zaccardi2018}.

Based on the discussed reasons, this study’s clinical goal is to predict A1C five years ahead using information provided by a continuous glucose monitoring device, what has never been explored in the literature. We select A1C as the outcome instead of FPG because A1C presented a much more reproducible laboratory biochemical measurement between different testing than FPG, and the response variable is subject to less measurement error \cite{selvin2007short}.

\subsection{RKHS models with missing data}

In the last decades, the statistical community has developed a vast amount of new methodology contributions to minimize the bias caused by missing data in covariates and response variable in various unsupervised and supervised modeling tasks \cite{londschien2020change, little2019statistical}. The contributions carried out under the paradigm of statistical learning or ML are sparser and more recent than in the field of Statistics; see, for example \cite{pmlr-v119-muzellec20a}.

In this paper, we restrict our attention to the case that the only variable with missing entries is the response. This situation is common in many longitudinal studies where the lack of data retrieval in some patients’ follow-up is frequent, and there is no information about the evolution of several outcomes related to them in different periods. 


We chose the RKHS learning paradigm to tackle our missing data problem. The primary rationale for this decision is that estimators have the optimal non-parametric convergence rate \cite{stone1982optimal} under certain hypotheses. For example, when data live in a low-dimensional manifold or a very smooth functional space. However, methods remain valid with heterogeneous complex data \cite{borgwardt2006integrating, muandet2016kernel} as graphs or curves that take values on a continuum, as in our case concerning the distribution-functional representation of glucose profiles built through the concept of glucodensity \cite{matabuena2020glucodensities}.

There is available methodology in this context only on predictive models for performing regression to the best of our knowledge. We introduce new methods for statistical independence testing, variable selection, and inference on the uncertainty of new predictions. Liu and Goldberg \cite{liu2020kernel} proposes a Kernel Ridge Regression estimator for either the case in which the missing data mechanism is handled by propensity score via IPW (\textit{Inverse Probability Weighting}) estimator or a double-robust approach \cite{tsiatis2007semiparametric, bang2005doubly}. Next, we summarize the state-of-the-art of RKHS-based techniques with complete data. Kernel means embedding \cite{muandet2016kernel, gretton2012kernel} is a powerful tool to build different statistics to contrast equality between probability distributions, measure statistical independence in RKHS spaces, and capture a broad interest between practitioners of the ML community and, in particular, of kernel methods. Variable Selection in RKHS spaces allows identifying the best subset of predictors without assuming any functional form underlying the covariates' dependence structure and the predictor variables. Two different strategies were considered: i) Learning the gradient of the conditional mean function  \cite{yang2016model}; ii) maximize the norm of the conditional correlation operator \cite{chen2017kernel}. 

Finally, measuring the uncertainty of predictions is an essential task to support clinical decision-making through these models. The previous issue can be done, for example, through the construction of a confidence interval that contains the real value with a certain probability margin. In a set up of complete data, this problem can be addressed with conformal inference \cite{shafer2008tutorial}. At this moment, no methodology is available with missing data. However, we exploit recent advances on causal inference in combination with this area \cite{lei2020conformal}, and we adapt our methodology using the connection between these two areas \cite{ding2018causal}.

\subsection{Summary of the results}\label{sec:resumenresults}

Let $\{(X_i, Y_i, R_i)\}_{i=1}^{n}$ be a independent random sample of a random vector $(X, Y, R)$ taking values in $V\times \mathbb{R}\times \{0,1\}$, where $V$ can denote any set as general as we want, for example, graphs or random functions. Let $X$ denote the covariates, $Y$ the response variable, and $R$ a binary random variable that indicates whether the response is missing or not, which we assume to be distributed according to the probability law $\pi (x)= \mathbb{P}(R=1|X=x)$, which depends on the covariates $X$. For this, we suppose that $\quad R  \perp Y| X$. In MAR missing data mechanism, $R\sim Bernoulli(p)$, where $p$ is the probability that the event "missing response datum" happens, and this event is independent of the values that the covariates take.

Therefore, we propose a new data analysis framework in different prediction tasks when some  $Y_i’s$ are not observed. In particular, $Y_i$  is missing  if $R_i=0$ $(i=1,\ldots, n)$.  Below, we discuss the modeling tasks that we address together with new different methodological contributions.

\begin{itemize}
	\item \underline{Independence statistical testing}:  A cornerstone problem in statistics, epidemiology, and in a general setting of data analysis is testing statistical independence between random variables X and Y. In this case; we carried on a hypothesis test using the sample $\{(X_i,Y_i)\}_{i=1}^{n}$, and we check if there exist any evidences that $X$ and $Y$ are not independent,  i.e., we can reject or not a null hypothesis $H_0: \quad X  \perp Y$ according to whether the statistic value belongs to the rejection region. To do this, we must calibrate the test under the null hypothesis to determine what results are expected to happen with a certain probability if the null hypothesis holds. In our concrete case, we have to take into account the effects of the mechanism of missing data in the response variable $Y$ in order to the design the test and to calibrate the null distribution, which is determined by the behavior of the previous function $\pi(\cdot)$. In the first case, we propose a methodology to deal with this problem based on kernel mean embeddings, which is valid when the covariates vector and response live in a separable Hilbert space. In addition, we introduce a new bootstrap procedure to perform test calibration, adapted to kernel mean embeddings.
	

	\item \underline{Variable selection}: Consider the mean regression problem: 
	\begin{equation}\label{eqn:1}
	Y= m(X)+\epsilon,	
	\end{equation}
	
where $\epsilon$ is a random error of mean zero and $X= (X^{1},\cdots, X^{p})$ is a random vector composed of $p$ covariates. In many problems, it is important to identify the subset of covariates $I\subset \{X^{1},\cdots, X^{p}\}$ that has an impact on the prediction $Y$. The previous problem is remarkable primarily because of the next two factors: i) to achieve parsimonious predictive models that generalize well with the new cases; ii) to discover the genuine causes associated with diseases or patient prognosis. We will modify the algorithm of Lei Yang, Shaogao Lv, Junhui Wang, so it remains valid when the response is missing \cite{yang2016model}.

	\item \underline{Prediction and inference}: The ultimate goal of any predictive task is always to explain the relationship between the variables $Y$ and $X$, for example, according to the model defined in equation \ref{eqn:1}. Here, with real-world data, $m(\cdot)$, can have any functional shape, although it is also common to restrict it to a parametric, semi-parametric form (e.g., additive structure), or assume that the conditional mean function lives in a smooth function space. Furthermore, it is crucial to measure the predictions’ uncertainty and give a region of probability containing the real value with an appropriate level of confidence. An appropriate level can be $90\%$ to control and secure the reliability of the results returned by the algorithm with a substantial margin of probability. In the prediction task, we will use the kernel ridge regression model proposed by Liu and Goldberg \cite{liu2020kernel}. However, using the theory of linear regression, we will calculate the leave-one-out cross-validation regularization parameter efficiently and take into account the missing data mechanism. It is important to note that this class of models’ regularization parameters largely determine the model performance, as evinced in some relevant recent papers \cite{liang2020just, hastie2019surprises, bartlett2020benign}. Additionally, using advances in conformal inference recently exploited in causal inference \cite{lei2020conformal}, we will obtain regions that have good finite sample coverage.
\end{itemize}

As for the glucose prediction clinical study case, our main contributions are the following:

\begin{itemize}
	
\item We identify several markers associated with the evolution of glucose five years ahead. 


\item With the aim to optimize medical decisions, we provide a predictive algorithm to forecast expected A1C values five years ahead. We use as covariates individual’s glycemic status and other clinical variables.

\item We interpret and discuss the residuals and the predictive capacity of the different models from the clinical point of view, providing interpretable clinical phenotypes in which future forecasts will have a large uncertainty.
	
\end{itemize}

\section{Missing data models}

\subsection{Testing statistical independence}\label{sec:independence}

Kernel mean embeddings \cite{gretton2007kernel, muandet2016kernel}, or the equivalent \textit{distance correlation} in the statistical community  \cite{szekely2007measuring,szekely2017energy,sejdinovic2013equivalence} are among the most extensive and general methodologies in the case of complete data to test statistical independence. Subsequently, let us introduce some elementary background over the previous distances/transformations before explaining the extension that the response variable can be missing.

Consider an arbitrary $H_{X}$ $RKHS$ associated with the random variable $X$,  which is uniquely determined by a positive definite symmetric kernel $K_{X}: V\times V\to \mathbb{R}^{+}$- with $V$ an arbitrary set, that satisfies the following two conditions: i) $K_{X}(\cdot,x)\in H_{X}$, ii) $<f,K_{X}(\cdot,x)>= f(x)$ $ \forall f\in H_{X}$.  Given a $V$-valued random variable $X$ with probability measure $P_{X}$, the kernel mean embedding of $X$ is defined as the function  $\phi_{X}: s \in V \mapsto \int_{V\times V}^{} K_{X}(s,x) dP(dx)=E_{X\sim P_{X}  }(K_{X}(\cdot, X))\in H_{X}$. Roughly speaking,  $\phi_{X}(\cdot)$ embeds the data in a new separable Hilbert space that is  typically infinite dimensional. In the following, we suppose that all used kernels  $K(\cdot,\cdot)$  are \textit{characteristic}, an important property that guarantees the ability to characterize independence or equality in distribution with this methodology against all alternatives. Concretely,  we say that $K_{X}(\cdot,\cdot)$ is a characteristic kernel if $\phi_{X}(\cdot)$ is injective \cite{simon2018kernel}.\\

Now, let $Y$ be another random variable that, for the sake of simplicity, we will assume to take values in $\mathbb{R}$ as in the problem definition above, which we will interpret as the \textit{response variable}. By definition, testing the null hypothesis $H_{null}: X \perp Y$ is equivalent to testing $H_{null}: P_{X, Y}= P_{X}P_{Y}$, that is, the  distribution probability measure  is expressed as a product of marginals measures. With this aim in mind, we introduce some extra notation. We denote by $\phi_{Y}(\cdot)$, $\phi_{X, Y}(\cdot)$ the kernel mean embeddings of $Y$ and the bivariate random variable $(X,Y)$ that depends on kernel $K_{(X,Y)}$. We must note, firstly, that $X$ and $Y$ can have different dimensions. In addition, it is natural to consider in  $V\times \mathbb{R}$ the RKHS space $H_{X}\otimes H_{Y}$, where it makes sense to define the global kernel as $K_{(X,Y)}(x,y)(x’,y’)= K_{X}(x,x’)K_{Y}(y,y’)$ $\forall (x,x’)\in V\times V$ and $(y,y’)\in \mathbb{R} \times \mathbb{R}$ and $\otimes$ denote tensor product. Then, a natural way of testing independence is measuring the distance between the functions  $\phi_{X, Y}(\cdot)$ and  $\phi_{Y}(\cdot) \otimes \phi_{X}(\cdot)$. More specifically, we define the Hilbert-Schmidt independence criterion (HSIC) between $P_{X, Y}$ and $P_{X}P_{Y}$ as

\begin{equation}\label{eqn:HSIC}
HSIC(P_{X, Y},P_{X}P_{Y})= ||\phi_{X, Y}- \phi_{X} \otimes \phi_{Y}||^{2}_{H_{X}\otimes H_{Y}}
\end{equation}

Expanding Equation \ref{eqn:HSIC} we have

\begin{equation} \label{eqn:HSIC1}
\begin{split}
& HSIC(P_{X, Y},P_{X}P_{Y}) = \\
& = <\phi_{X, Y}- \phi_{X} \otimes \phi_{Y},\phi_{X, Y}- \phi_{X} \otimes \phi_{Y} >_{H_{X}\otimes H_{Y}}\\ &=<\phi_{X, Y}, \phi_{X, Y}>_{H_{X}\otimes H_{Y}}+<\phi_{X} \otimes \phi_{Y},\phi_{X} \otimes \phi_{Y}>_{H_{X}\otimes H_{Y}}\\&-2<\phi_{X, Y},\phi_{X} \otimes \phi_{Y}>_{H_{X}\otimes H_{Y}}.
\end{split}
\end{equation}

and using properties of $RKHS$ and Fubini's theorem, we get 

\begin{equation}\label{eqn:HSIC2}
\begin{split}
& HSIC(P_{X, Y},P_{X}P_{Y}) = \\
& = E_{(X,Y,X^{’}, Y^{’})}(K_{X}(X, X’)K_{Y}(Y, Y’))\\
& + E_{(X,X^{’})}(K_{X}(X, X’)) E_{(Y,Y^{’})}(K_{Y}(Y, Y’))\\
& -2E_{(X,Y)}(E_{X’}(K_{X}(X,X’))E_{Y’}(K_{Y}(Y,Y’) )).
\end{split}
\end{equation}

Here $X',Y'$ are iid. copies of random variables $X,Y$.

To understand the HSIC procedure well, it is essential to remark that this procedure consists only of calculating the squared distance between two mean functions in the appropriate RKHS space, which was transformed into original data to capture all distributional differences between the involved random variables.

In practice, only a sample $\{(X_i,Y_i)\}_{i=1}^{n}$ is observed. Therefore, we must replace the population mean by sample mean defined through its empirical distribution. Then, the empirical estimator of the Hilbert-Schmidt independence criterion estimator is given  by

\begin{equation}\label{eqn:HSIC3}
\begin{split}
& \widehat{HSIC}(\hat{P}_{X, Y},\hat{P}_{X}\hat{P}_{Y})= \\
& = \frac{1}{n^{2}} \sum_{i=1}^{n} \sum_{j=1}^{n} (K_{X}(X_i, X_j)K_{Y}(Y_i, Y_j))\\
& +\frac{1}{n^{2}} \sum_{i=1}^{n} \sum_{j=1}^{n} K_{X}(X_i, X_j) \sum_{i=1}^{n} \frac{1}{n^{2}} \sum_{i=1}^{n} \sum_{j=1}^{n} K_{Y}(Y_i, Y_j)\\
&-\frac{1}{n^{3}} \sum_{i=1}^{n}\sum_{j=1}^{n}\sum_{k=1}^{n} K_{X}(X_i,X_j)K_{Y}(Y_i,Y_k).
\end{split}
\end{equation}
With MNAR data, we observe $\{(X_i, Y_i, R_i)\}_{i=1}^{n}$ and we have to estimate the missing data mechanism, that is given by the function $\pi(\cdot)= \mathbb{P}(R=1|X=\cdot)$. Several procedures were proposed in the literature for this aim such as logistic regression, lasso, random forest, or a model ensemble as Super Learner among others \cite{van2007super}. Afterwards, we re-weight the dataset, taking into account how difficult it is to observe the response of the $i$th datum. In particular, we define the weight associated with the $i$th datum $w_i$ via inverse probability weighting (IPW) estimator, as

\begin{equation}\label{eqn:pesos}
w_i= \frac{R_i }{n \pi(X_i)} \hspace{0.4cm} (i=1,\cdots,n).
\end{equation}
We define the normalized-weight of $w_i$, as
\begin{equation}\label{eqn:pesos2}
w_i^{*}= \frac{w_i}{\sum_{i=1}^{n} w_i} \hspace{0.4cm} (i=1,\cdots,n).
\end{equation}

We denoted the estimated $i$th-weigh as $\hat{w}_i$ and $\hat{w}^{*}_{i}$ respectively, after estimate $\hat{\pi}(\cdot)$.

To get an estimator of HSIC with missing data, it is enough to replace the uniform weight $\frac{1}{n}$ of the empirical distribution with the normalized weights $\hat{W}^{*}=(\hat{w}_1^{*},\cdots,\hat{w}_n^{*})$ in the Equation \ref{eqn:HSIC3}. Concretely, we have,
\begin{equation}\label{eqn:HSIC4}
\begin{split}
&\widehat{HSIC}(\hat{P}_{X, Y},\hat{P}_{X}\hat{P}_{Y})= \\
&\sum_{i=1}^{n} \sum_{j=1}^{n} \hat{w}^{*}_{i} \hat{w}^{*}_{j} (K_{X}(X_i, X_j)K_{Y}(Y_i, Y_j))\\
&+\sum_{i=1}^{n} \sum_{j=1}^{n} \hat{w}^{*}_{i} \hat{w}^{*}_{j} K_{X}(X_i, X_j) \sum_{i=1}^{n} \sum_{j=1}^{n} \hat{w}^{*}_{i} \hat{w}^{*}_{j} K_{Y}(Y_i, Y_j)\\
&-\sum_{i=1}^{n}\sum_{j=1}^{n}\sum_{k=1}^{n} \hat{w}^{*}_{i} \hat{w}^{*}_{j} \hat{w}^{*}_{k} K_{X}(X_i,X_j)K_{Y}(Y_i,Y_k).
\end{split}
\end{equation}

Calibration under the null hypothesis with the precedent statistic is not trivial, and the permutation approach is generally not valid. We propose a bootstrap approach for this case based on simple Efron's bootstrap \cite{efron1994introduction}, which remains valid with glucodensities: several bootstrap procedures cannot deal with complex constrained distributional objects that do not live in vector spaces.

Under the null hyphotesis $H_{null}: P_{X,Y}= P_{X}P_{Y}$, then $\phi_{X, Y}(\cdot)- \phi_{Y}(\cdot)\otimes \phi_{X}(\cdot)=0(\cdot)$. So,
\begin{equation}\label{eqn:bootstrap1}
\begin{split}
&\widehat{HSIC}(\hat{P}_{X, Y},\hat{P}_{X}\hat{P}_{Y})= \\ 
&< \hat{\phi}_{X, Y}-\hat{\phi}_{Y} \otimes \hat{\phi}_{X}, \hat{\phi}_{X, Y}-\hat{\phi}_{Y} \otimes\hat{\phi}_{X}>_{H_{X}\otimes H_{Y}}= \\
&< \hat{\phi}_{X, Y}- \phi_{X, Y} + \phi_{Y} \otimes \phi_{X} - \hat{\phi}_{Y} \otimes \hat{\phi}_{X}, \\
&\hat{\phi}_{X, Y}- \phi_{X, Y} + \phi_{Y} \otimes \phi_{X} - \hat{\phi}_{Y} \otimes \hat{\phi}_{X}>_{H_{X}\otimes H_{Y}}.\\ 
\end{split}
\end{equation}
Then, a natural bootstrap procedure that allows estimating the $p$-value of the independence testing problem can be as follows:

\begin{enumerate}
	
	\item Select randomly with replacement and equal probability $n$ elements from the original sample $D=\{(X_i,Y_i, R_i)\}_{i=1}^{n}$ $m$ times. We denote by $D^{j^{*}}=\{(X^{j^{*}}_i,Y^{j^{*}}_i, R^{j^{*}}_i)\}_{i=1}^{n}$ $(j=1,\cdots m)$, the $j$th random sample obtained.

	\item Calculate $\widehat{HSIC^{j^{*}}}(\hat{P}_{X, Y},\hat{P}_{X}\hat{P}_{Y})$ as

\begin{equation}\label{eqn:bootstrap2}
\begin{split}
	&\widehat{HSIC^{j^{*}}}(\hat{P}_{X, Y},\hat{P}_{X}\hat{P}_{Y})=\\	& < \hat{\phi}_{X, Y}- \hat{\phi}^{j^{*} }_{X, Y} + \hat{\phi}^{j^{*}}_{Y} \otimes \hat{\phi}^{j^{*}}_{X} - \hat{\phi}_{Y} \otimes \hat{\phi}_{X}, \\
	&\hat{\phi}_{X, Y}- \hat{\phi}^{j^{*} }_{X, Y} + \hat{\phi}^{j^{*}}_{Y} \otimes \hat{\phi}^{j^{*}}_{X} - \hat{\phi}_{Y} \otimes \hat{\phi}_{X}(\cdot)>_{H_{X}\otimes H_{Y}} \\	& (j=1,\cdots m),
	\end{split}
\end{equation}

	where $\hat{\phi}^{j^{*} }_{X, Y}(\cdot)$, $\hat{\phi}^{j^{*}}_{X}(\cdot)$ and $\hat{\phi}^{j^{*}}_{X}(\cdot)$ are the kernel mean embedding estimated with the bootstrap-sample $D^{j^{*}}=\{(X^{j^{*}}_i,Y^{j^{*}}_i, R^{j^{*}}_i)\}_{i=1}^{n}$.

	\item Estimate the $p-$value as

\begin{equation}\label{eqn:bootstrap3}
\begin{split}
    	&\text{p-value}=\\ 
	&\frac{1}{m} \sum_{j=1}^{m}I( \widehat{HSIC^{j^{*}}}(\hat{P}_{X, Y},\hat{P}_{X}\hat{P}_{Y}) \geq   \widehat{HSIC}(\hat{P}_{X, Y},\hat{P}_{X}\hat{P}_{Y})).
	\end{split}
\end{equation}
\end{enumerate}

Using some standard tools of empirical process theory \cite{van1996weak, van2000applications} , we can establish the bootstrap's consistency with missing data in this framework. 	We introduce specif details in the Appendix, Section \ref{sec:apendice}.



\subsection{Variable selection}\label{sec:selection}

Learning in RKHS allows to detect complex dependence relations between $X= (X^{1},\cdots,X^{p})$ and $Y$. In particular, given the regression model defined in the Equation \ref{eqn:1}, we can select the influential truth variables using the methods-provided in this framework in spite of the fact that we do not have prior information about the functional form of $m(\cdot)$ as in most modern applications.

Suppose that $m(\cdot)$ is a differentiable function and let $g(x)= \nabla m(x)= (\frac{\partial m}{\partial x^{1}}(x),\cdots, \frac{\partial m}{\partial x^{p}}(x))$ be its gradient avaliated in the point $x\in V$. We know that $x^{i}$ in Equation \ref{eqn:1} is an irrelevant predictor if $\frac{\partial m}{\partial x^{i}}(x)=0$ $\forall x\in V$, namely, if the partial derivative along direction $i$ is null then the $i$th covariate does not provide information about the geometry of the prediction function $m(\cdot)$. Then, a naïve approach to select the best subset of variables non parametrically is to learn the gradient and use the value of the norm of that estimation as a criterium. We note that there exists vast research on finding the best subset of covariates in the setup of linear regression model, where currently is an active research area, see for example this following pieces of contemporary work \cite{hazimeh2018fast,bertsimas2016best}.

In a small neighbourhood of $x$, $A_{x}$, we know in virtue of Taylor's theorem that $Y(x)= m(x)+\epsilon \approx m(z)+g(z)^{T}(x-z)+\epsilon$ $\forall z\in A_{x}$. As $m(z)+\epsilon= Y(z)$, then $Y(x) \approx Y(z)+g(z)^{T}(x-z)$ or $Y(x)-Y(z)\approx g(z)^{T}(x-z)$ what implies $E(Y(x)-Y(z))=g(z)^{T}(x-z)$ $\forall z\in A_{x}$ if the volume of $A_x$ is small enough. Then, the  squared error of the true gradient $g(\cdot)$ function can be constructed:

\begin{multline}\label{eqn:sel1}
 \int_{(V\times \mathbb{R})(V\times \mathbb{R})}^{} \omega(x,z)(y-v-g(z)(x-z))^{2}\cdot\\
 \cdot dP_{X,Y}(dx,dy)dP_{X’,Y’}(dz,dv)
\end{multline}

where $\omega(x,z)$ is an appropriate weight function that is zero when $z$ is not in the same neighborhood as  $x$. In addition, with $X’,Y’$, we denote independent and random variables distributed as $X,Y$, respectively.

With the loss-function defined in Equation \ref{eqn:sel1} in mind, the empirical estimator is naturally defined as
\begin{multline}\label{eqn:sel2}
\hat{g}(\cdot)= \arg \min_{g= (g^{1}, \dots g^{p}) \in  \underbrace{H\times \cdots \times  H}_{\text{p-times}}} \\ \sum_{i=1}^{n} \sum_{j=1}^{n} \omega(X_i,X_j)(Y_i-Y_j-g^{T}(X_i)(X_i-X_j))\\ +\lambda \sum_{i=1}^{p} ||g^{i}||^{2}.
\end{multline}

Here, $H$ denotes the selected RKHS associated in each coordinate of the gradient we aim to learn, and $\lambda$ is the smoothing parameter that avoids over-fitting in the modeling task if it is appropriately selected. 

If we observe Equation \ref{eqn:sel2}, the objective function is convex and by definition the solution of the minimization problem lies inside the cartesian product of RKHS, which is as well an RKHS. Then we can apply the classical Representer theorem \cite{scholkopf2001generalized} to guarantee that the solution for each component of gradient is defined explicitly by $\hat{g}^i(x) = \sum_{j=1}^{n} \hat{\alpha}^{i}_j K(X_j,x)$ $(i=1,\cdots,p)$, where $\hat{\alpha}^{i}_j$’s are real numbers. Replacing the functional form of a solution in Equation \ref{eqn:sel2}, the optimization problem turns to be:

\begin{multline}\label{eqn:sel3}
\hat{\alpha}= \arg \min_{ \alpha= (\alpha^i_j)^{i=1,\cdots,p}_{j=1,\cdots,n}\in \mathbb{R}^{n\times p} } \\\sum_{i=1}^{n} \sum_{j=1}^{n} (Y_i-Y_j- \sum_{k=1}^{p}\alpha^{j}_{i}K(X^{k}_j,X^{k}_i)(X^{k}_j,X^{k}_i))
\\ \sum_{i=1}^{p} \sum_{j=1}^{n} \sum_{k=1}^{n}  \lambda \alpha_j^{i} K(X^{i}_j, X^{i}_k)\alpha_k^{i}
\end{multline}

Again, we can adapt the algorithm to MNAR data with the IPW estimator. In particular, we have:

\begin{multline}\label{eqn:sel4}
\hat{\alpha}= \arg \min_{\alpha= (\alpha^i_j)^{i=1,\cdots,p}_{j=1,\cdots,n}\in \mathbb{R}^{n\times p} } \\\sum_{i=1}^{n} \sum_{j=1}^{n} \hat{\omega}^*_{i}  \hat{\omega}^*_{j} \omega(X_i,X_j) (Y_i-Y_j- \sum_{k=1}^{p}\alpha^{j}_{i}K(X^{k}_j,X^{k}_i)(X^{k}_j,X^{k}_i))
\\+  \sum_{i=1}^{p} \sum_{j=1}^{n} \sum_{k=1}^{n}  \lambda \alpha_j^{i} K(X^{i}_j, X^{i}_k)\alpha_k^{i},
\end{multline}

where $\hat{\omega}^*_{i}$ is denotes the $i$th normalized weight according to the missing data mechanism, see Equation \ref{eqn:pesos2} for more details.

In practice, the smoothing parameter is selected via cross-validation, and the precedent optimization problem is solved with specific optimization gradient techniques of group Lasso \cite{ida2019fast,yang2015fast} taking into account the multidimensional block structure of $\hat{\alpha}^{i}_j$'s $(i=1,\cdots,p)$, that is, if any  $\hat{\alpha}^{i}_j=0$ $(j=1,\cdots,n)$, then $\hat{g}^{i}(\cdot)=0$. For specific details about this procedure, we refer the reader to the original paper, which concerns complete data. .

\subsection{Predictive models}\label{sec:predictive}

Given a sample $\{(X_i,Y_i)\}_{i=1}^{n}$, linear ridge regression is based on solving the following optimization problem:

\begin{equation}\label{eqn:pred1}
\hat{\beta}= \arg \min_{\beta \in \mathbb{R}^p} \sum_{i=1}^{n} (Y_i-<X_i,\beta>)^{2}+\lambda ||\beta||^{2}_2
\end{equation} 

which is given by $\hat{\beta} = (X^{T}X+\lambda I)^{-1} X^{T} Y$ where $X= \begin{pmatrix}
X_1 \\ 
\vdots \\
X_n
\end{pmatrix}$,  
$Y= \begin{pmatrix}
Y_1 \\ 
\vdots \\
Y_n
\end{pmatrix}$ and $\lambda>0$ is the smoothing parameter of regularization term. \\

Let $H_{X}$ be a RKHS space with kernel $K_{X}(\cdot, \cdot)= <\phi(\cdot), \phi(\cdot)>$. Then, if in Equation \ref{eqn:pred1} we  transform each $X_i$ into $\phi(X_i)$ $(i=1,\dots,n)$  and  suppose that $\beta= \sum_{i=1}^{n} \alpha_i \phi(X_i)$, we can obtain a solution to the ridge problem with a similar structure only changing the usual dot product by the inner product of the selected RKHS space. In particular, we have $\hat{\alpha} = (K+\lambda I)^{-1} Y$, where 
\vspace{0.7cm}
\begin{center}
	$K= \begin{pmatrix}
	K_{X}(X_1,X_1) & \cdots & K_{X}(X_1,X_n) \\ 
	\vdots & \ddots & \vdots \\
	K_{X}(X_n,X_1) & \dots &  K_{X}(X_n,X_n)
	\end{pmatrix}$. 
	\vspace{0.7cm}
\end{center}

Following Liu and Goldberg \cite{liu2020kernel}, they propose  two different estimator when the response is missing in this setting. In both cases, the solution has the same close-expression given by  Representer Theorem \cite{scholkopf2001generalized} in the usual way.   First, they handle missing data mechanism via IPW estimator an obtain $\hat{\alpha}= (\lambda I+ KW)^{-1} W Y$. Second, through doubly robust estimation  that combines  preliminary imputation  of missing response with IPW estimator, they get

 \begin{equation}\label{eqn:regmiss}
\hat{\alpha}= (K+\lambda I)^{-1}(WY+ (I-W)\mu(X)), where
\end{equation}

\vspace{0.7cm}

$W= \begin{pmatrix}
w_1 & 0 & 0 \\ 
0 & \ddots & 0 \\
0 & 0 &  w_n
\end{pmatrix},$  denotes a diagonal matrix that contains the weights $w_i's$ (see Equation \ref{eqn:pesos}) and  \vspace{0.7cm} \vspace{0.7cm}
 $\mu(X)= \begin{pmatrix}
\mu(X_i) \\
\vdots \\
\mu\left(X_n\right)
\end{pmatrix}$, where $\mu(\cdot)$ denotes the imputation function.

Doubly robust estimators have optimal asymptotic variance when their weights  $w_i’s$ and their imputation model are correctly specified, and only one of these approaches needs to be correctly specified to achieve consistency. However, if any of them fails, then the model performance can deteriorate dramatically with finite sample \cite{kang2007demystifying}. When both models are misleadingly specified, no gain may be obtained with this more sophisticated approach \cite{vermeulen2015bias,kang2007demystifying}.

An essential issue in the performance of these models and that has received considerable attention in recent years is the smoothing parameter’s impact on the generalization of the models, which is strongly connected with the interpolation problem in RKHS space with minimum norm. In our setting, we have selected the smoothing parameter through \textit{leave-one-out} cross-validation, calculating the explicit expression using linear regression theory with the missing data formulae of estimators. Further details about the recent theoretical advances in this field are available in the following papers \cite{liang2020just, hastie2019surprises, bartlett2020benign}.


To end this subsection, we introduce a specific algorithm to perform conformal inference based on \cite{lei2020conformal}, that allows to provide an interval that contains the response with a confidence level $1-\alpha$ for new observation $X_{n+1}$. 

We randomly split the data $\{(X_i, Y_i, R_i)\}_{i=1}^{n}$ without replacement in two samples $X^{training}=\{(X^{training}_i, Y^{training}_i, R^{training}_i)\}_{i=1}^{n_1}$, $X^{test}=\{(X^{test}_i, Y^{test}_i, R^{test}_i)\}_{i=1}^{n_2}$ of size $n_1$, $n_2$ respectively, with $n_1+n_2=n$. The steps of the algorithm are summarized as below:

\begin{enumerate}
	
	\item Using  $X^{training}$, fit the mean regression function $\hat{m}(\cdot)$ according the method provided in the Equation \ref{eqn:regmiss}. 
	
	\item  For all $i=1,\cdots n_2$ with $R^{test}_{i}=1$, define the following non-conformal measure:
	
	$\hat{\epsilon}_{i}= |Y^{test}_i-\hat{m}(X^{test}_i)|.$ 
	
	\item We estimate the empirical distribution $\hat{F}_{n_{2}+1}^{\epsilon}$ using the previous residuals and representing the theoretical residual of observation $X_{n+1}$ with the artificial value of infinite. For this task, we use the IPW estimate with weights defined in Equations (\ref{eqn:pesos},\hspace{0.02cm} \ref{eqn:pesos2}) and the function $\hat{\pi}^{training}(\cdot)$ calculated in Step $1$, where we must incorporate also the weight of $X_{n+1}$, $\hat{w}_{n_{2+1}}$.
	
	\item With $\hat{F}_{n_{2}+1}^{\epsilon}$, calculate the $1-\alpha$ quantile that we denote with $\hat{q}_{1-\alpha}$.
	
	\item Return $[\hat{m}(X_{n+1})-\hat{q}_{1-\alpha}, \hat{m}(X_{n+1}) +\hat{q}_{1-\alpha}]$ as the searched interval.

\end{enumerate}

\subsection{Handling multiple sources with a kernel}\label{sec:multkernel}

RKHS modeling is a powerful data analysis paradigm that allows efficient data analysis of different nature simultaneously \cite{borgwardt2006integrating}. To do this, the critical point is to select a suitable kernel that accurately captures the differences and specific characteristics of each of the information sources examined. In our particular case, we have a continuous probability distribution, multidimensional data, and categorical data, $X= (X^{gluco}, X^{mult}, X^{categ})\in V^{gluco} \times V^{mult} \times V^{categ} $.  We know that the Gaussian kernel with the standard Euclidean distance is a characteristic and universal kernel with vectorial real-data. Moreover, we can show that Gaussian-Kernel conserves those mentioned above, considering the set of continuous density functions endowed with $L2-$Wasserstein geometry. Under these conditions, we have theoretical guarantees that we can approximate a large variety of functional forms $m(\cdot)$ in each model fitted individual. However, we want to detect possible interactions between different data sources, and for this, we must build a proper global kernel. Based on the connection between the defined positive kernel and the negative type metrics \cite{lyons2013distance} \cite{berg1984harmonic,sejdinovic2013equivalence}, we know some properties that allow us to build a simple kernel that integrates the three sources. To this end, in our setting, we can define a global Gaussian kernel as

\begin{multline}\label{eqn:gauss}
K_{X}(x=(x^{gluco}, x^{mult}, x^{categ}),y=(y^{gluco}, y^{mult}, y^{categ})) \\= e^{-(a\frac{||x^{gluco}-y^{gluco}||^2}{\sigma^2 _{gluco}}+b\frac{||x^{mult}-y^{mult}||^2}{\sigma^2_{mult}}+c\frac{||x^{categ}-y^{categ}||^2}{\sigma^2_{categ}})} \\ \forall x,y\in V
\end{multline}

where $a,b,c, \sigma_{gluco}, \sigma_{mult} ,\sigma_{categ} >0$ and we assume for the sake of simplicity that  $(a,b,c)\in \{(a,b,c)\in \mathbb{R}^{3}: a+b+c=1 \hspace{0.2cm} \text{and} \hspace{0.2cm} 0\leq a \leq 1,  0\leq b \leq 1, 0\leq c \leq 1 \}$.

 A more refined variety of strategies to build a global kernel can be found in \cite{gonen2011multiple}.

\subsection{Selection parameter in Gaussian kernels}\label{sec:gauskernel}

 It is well known that kernel parameter tuning is more sensitive than the choice itself \cite{scholkopf2002learning} among a family of kernels with the property of being characteristic or universal \cite{simon2018kernel} in different modeling tasks. For this reason, here, we have restricted ourselves to the Gaussian kernel. Besides, there is an available rule to select the bandwidth parameter $\sigma>0$ as the heuristic median \cite{garreau2017large}.

Let $\{X_i= (X_i^{gluco}, X_i^{mult}, X_i^{categ}) \}_{i=1}^{n}$ be the sample and we build the kernel matrix $(K)_{i,j}= K(X_i,X_j),$ $i,j=1,\dots, n$ where $K(X_i, X_j)= e^{-\frac{||X_i-X_j||^{2}}{\sigma^{2}}}$. The heuristic median rule $\sigma>0$ is defined as:

\begin{equation}\label{eqn:heuristic}	
\sigma= \sqrt{median\{||X_i-X_j||^{2}: 1\leq i < j\leq n \}}
\end{equation} 

In the sense of \cite{reddi2014decreasing}, we suggest to find the optimal kernel bandwidth parameter $\sigma^{*}$ in a grid of points of the form  $\sigma^{*}= \sigma^{\gamma}$ with $\gamma\in (0,3]$. In the setting of our global Gaussian kernel \ref{eqn:gauss}, we propose to use the heuristic median rule \ref{eqn:heuristic} for each data type with the grid strategy shown above to select $\sigma_{gluco}, \sigma_{mult} ,\sigma_{categ}$.  $a,b,c>0$ parameters are selected also in a grid  but in this case which belong to a 3-dimensional-simplex.

Finally, to incorporate the missing data mechanism in the kernel bandwidth estimator, we calculate the median through IPW estimator.

\section{AEGIS database analysis}

\subsection{Data Description}

The AEGIS population study conducted in the Spanish town of A Estrada (Galicia) aims to analyze the steady evolution of different clinical features such as longitudinal changes in circulating glucose in $1516$ patients over $10$ years. In addition, non-routinary medical tests such as continuous glucose monitoring are performed every five years on a randomized subset composed of $581$ patients. Table \ref{table:tablacaracteristicas} shows the basal characteristics of the $581$ continuous glucose monitored patients grouped by sex. The collected variables include age, glycosilated hemoglobin (A1C), fasting plasma glucose (FPG), insulin resistance (HOMA-IR), body mass index (BMI); along with glycemic variability metrics: continuous overall net glycemic action (CONGA), mean amplitude of glycaemic excursions (MAGE), mean of the daily differences (MODD). As we appointed before, A1C and FPG are the popular variables of diabetes diagnosis or control in the standard clinical routine. HOMA-IR or insulin resistance is an essential variable which is strongly connected to different cellular mechanisms of diabetes development \cite{rehman2016mechanisms}. Body mass index is a global variable of health status related to all mortality causes, progression of diseases, complications, and the onset of a wide multi-spectrum of diseases such as metabolic syndrome or diabetes. Finally, the different examined glucose variability metrics capture aspects of glucose values' oscillation along different periods. Glucose variability is a representative characteristic of glucose metabolism, being the third component of dysglycemia \cite{monnier2008glycemic}.

We can found specific details about how laboratory measurements and continuous glucose monitoring were performed \cite{gude2017glycemic,matabuena2020glucodensities}.

\begin{table}[ht!]
	\centering
	\scalebox{0.8}{
		\begin{tabular}{lll}
			\hline
			& Men $(n=220)$ & Women $(n=361)$ \\ 
			\hline
			Age, years & $47.8\pm 14.8$ & $48.2\pm14.5$ \\ 
			A1C, \% & $5.6\pm0.9$ & $5.5\pm0.7$ \\ 
			FPG $mg/dL$ & $97\pm23$ & $91\pm21$ \\ 
			HOMA-IR $mg/dL.\mu UI/m$ & $3.97\pm5.56$ & $2.74\pm2.47$ \\ 
			BMI $kg/m^2$ & $28.9\pm4.7$ & $27.7\pm5.3$ \\ 
			CONGA $mg/dL$  & $0.88\pm0.40$ & $0.86\pm0.36$ \\ 
			MAGE $mg/dL$ & $33.6\pm22.3$ & $31.2 \pm14.6$ \\ 
			MODD & $0.84\pm0.58$ & $0.77\pm0.33$ \\ 
			\hline
	\end{tabular}} 
	\caption{Characteristics of AEGIS study participants with CGM monitoring by sex. Mean and standard deviation are shown.	$BMI$ - body mass index; $FPG$ - fasting plasma glucose; $A1c$ - glycated haemoglobin; $HOMA-IR$ - homeostasis model assessment-insulin resistance; $CONGA$ - glycemic variability in terms of continuous overall net glycemic action; $MODD$ - mean of daily differences; $MAGE$ - mean amplitude of glycemic excursions.}
	\label{table:tablacaracteristicas}
\end{table}

We address the main challenge of forecasting the relationship between the predictors and the outcome five years ahead: A1C to be precise. $40$ \% data of this variable is missing. Therefore, we have to use specific missing data techniques that limit biases in the obtained results to be widely generalized and reproducible to other study populations. 

\subsection{Clinical questions}

With the values in Table \ref{table:tablacaracteristicas} \emph{together} that were collected through CGM data via glucodensity representation \cite{matabuena2020glucodensities}  (Figure \ref{fig:graf1}), we treat to answer the following clinical open problems:

\begin{enumerate}
	\item We want to study if there exists a statistical association between each predictor using the A1C five year ahead, as outcome. For this purpose, we use the Hilbert-Schmidt independence criterion that we propose in the context of missing data in Section \ref{sec:independence} together with a specific bootstrap approach that we design for such a task. Moreover, we study the clinical relevance of marginal associations with a bidimensional plot.
	
	\item  We selected the non-parametric subset of variables (Section \ref{sec:selection}) that best explains glucose changes using vectorial valued data of Table \ref{table:tablacaracteristicas}. With this analysis, we can obtain new clinical knowledge about this diabetes biomarker's role in changes of glucose metabolism.

	\item We assess the impact of introducing CGM data via glucodensities in the models predictive capacity. For this purpose, we fit two kernel ridge regression models (Section \ref{sec:predictive}): one that includes glucodensities and other which does not. In some patients who residuals are large for, we use conformal inference to measure the uncertainty and characterize patient phenotypes.

\end{enumerate}

\subsection{Results}

\subsubsection{Statistical univariate association between predictors and difference A1c-post vs A1c-pre}

Our aim is to study whether there is any evidence of univariate statistical association for normoglycemic patients between glucose variation measured by  $A1C_{5-years}-A1C_{initial}$ and the variables shown in Table \ref{table:tablacaracteristicas}. In this particular case, the underlying missing data mechanism is estimated using univariate logistic regression.

\begin{figure}[ht!]
	\centering
	\includegraphics[width=0.9\linewidth]{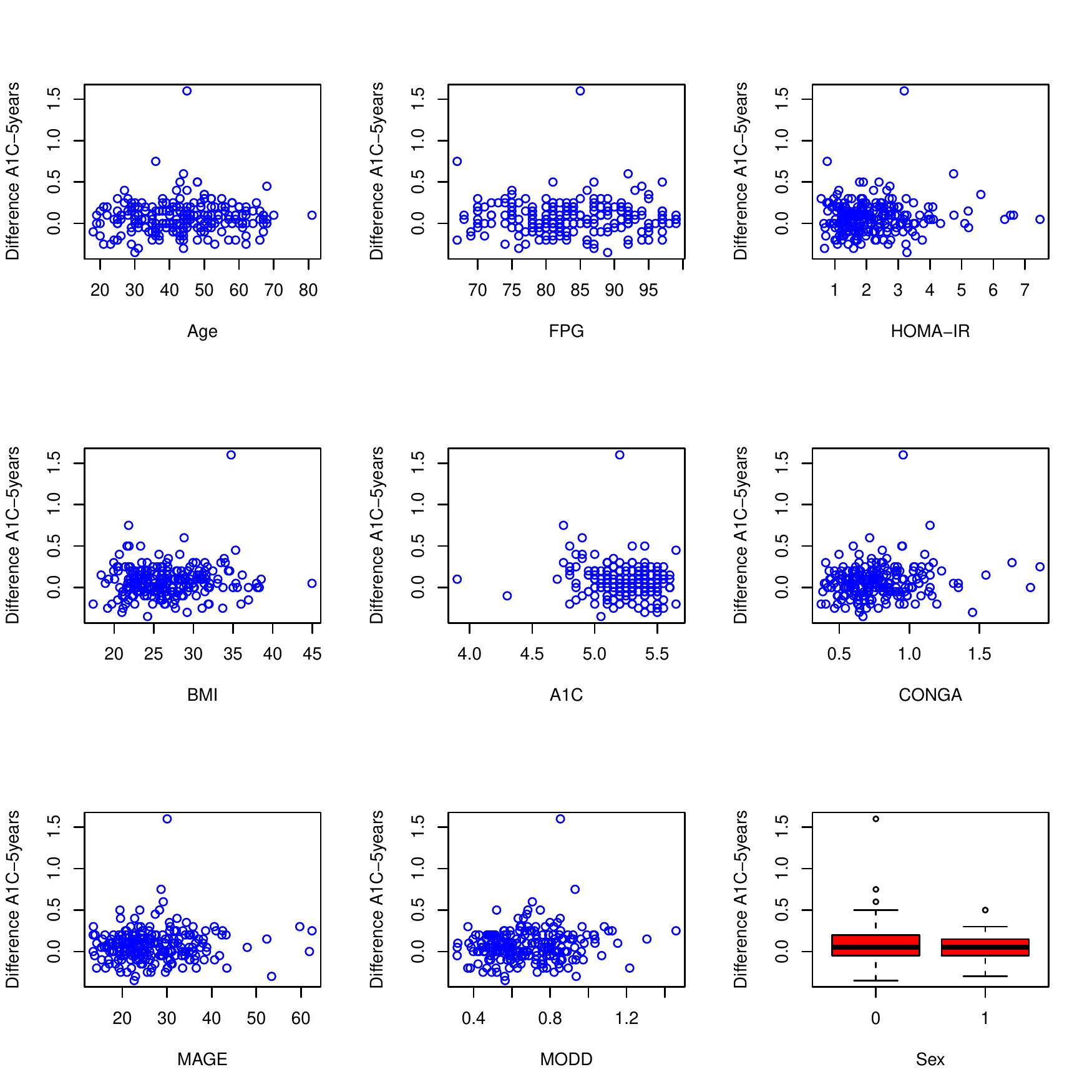}

\end{figure}
\begin{table}[ht!]
	\centering
	\begin{tabular}{lr}
		\hline
		Variable & $p-value$ \\ 
		\hline
		Age & $0.32$ \\ 
		Sex & $0.16$ \\ 
		FPG & $0.50$ \\ 
		HOMA-IR & $0.52$ \\ 
		BMI & $0.42$ \\ 
		A1C & $0.03$ \\ 
		CONGA & $0.24$ \\ 
		MAGE & $0.68$ \\ 
		MODD & $0.16$ \\ 
		Glucodensity & $<0.001$\\
		\hline
	\end{tabular}
\caption{Estimated raw p-values of A1C total variation vs each biomarker using the procedure defined in Section \ref{sec:independence} with normoglycemic patients. The figure shows the marginal dependence relation between examined variables}
	\label{fig:graf2}
\end{table}

Results in Table \ref{fig:graf2} show that the only statistically significant variables with a p-value less than 5\% are glucodensities and basal A1C. The plots above illustrate that the dependency relations among vectorial variables are weak, in case any of them held. Next, we are using multivariate models that exploit potential interactions between variables to improve association with changes in the A1C variable.

\subsubsection{Variable selection of vector-valued features with a non-parametric model}

We have adjusted the model defined in Section \ref{sec:selection} seeking the subset of variables most strongly associated with A1C values five years ahead. For this purpose, we  used $581$ patients, and we considered all the variables on Table \ref{table:tablacaracteristicas} except sex. In order to avoid overfitting and improving results reproducibility, we select model parameters using cross-validation. Finally, we estimate the underlying missing data mechanism via lasso logistic regression.

 
In this case, the variables selected by the algorithm are: Age, A1C, FPG, BMI, and MAGE. It is essential to note that the procedure used detected higher-order interactions between the covariates and the respective variable from  multivariate perspective. Moreover, in contrast to the previous section, both diabetic and non-diabetic patients have been analyzed.

\subsubsection{Kernel ridge regression prediction of future glucose values}

We fit two kernel ridge regression models with the goal of predicting A1C at five years ahead. The first includes non-CGM-variables A1C, FPG, Age, BMI as covariates, and the second one incorporates CGM data via glucodensities as well. Kernel selection and parameter tuning have been calibrated following the indications on Sections \ref{sec:multkernel} and \ref{sec:gauskernel}. The $R^2$ of the first model, according to missing data mechanism and by leave-one cross-validation calculated, is $0.56$, whereas in the second case, it is $0.66$. In Figure \ref{fig:grafico2pred}, we plot the residues against each value of basal A1c. In general, we can say that the most considerable residues are found in diabetic patients, while in other cases, the distribution of individuals residuals is heterogeneous depending on the patient's characteristics. 

These results show that when introduce in the models the glucodensity, CGM information can provide a piece of valuable extra knowledge on long-term glucose changes.

Figure \ref{fig:grafico3pred} depicts confidence intervals at a confidence level of 90 $\%$ after applying conformal inference methodology to measure the uncertainty of the predictions performed by the regression model.

\begin{figure}[ht!]
	\centering
	\includegraphics[width=0.9\linewidth]{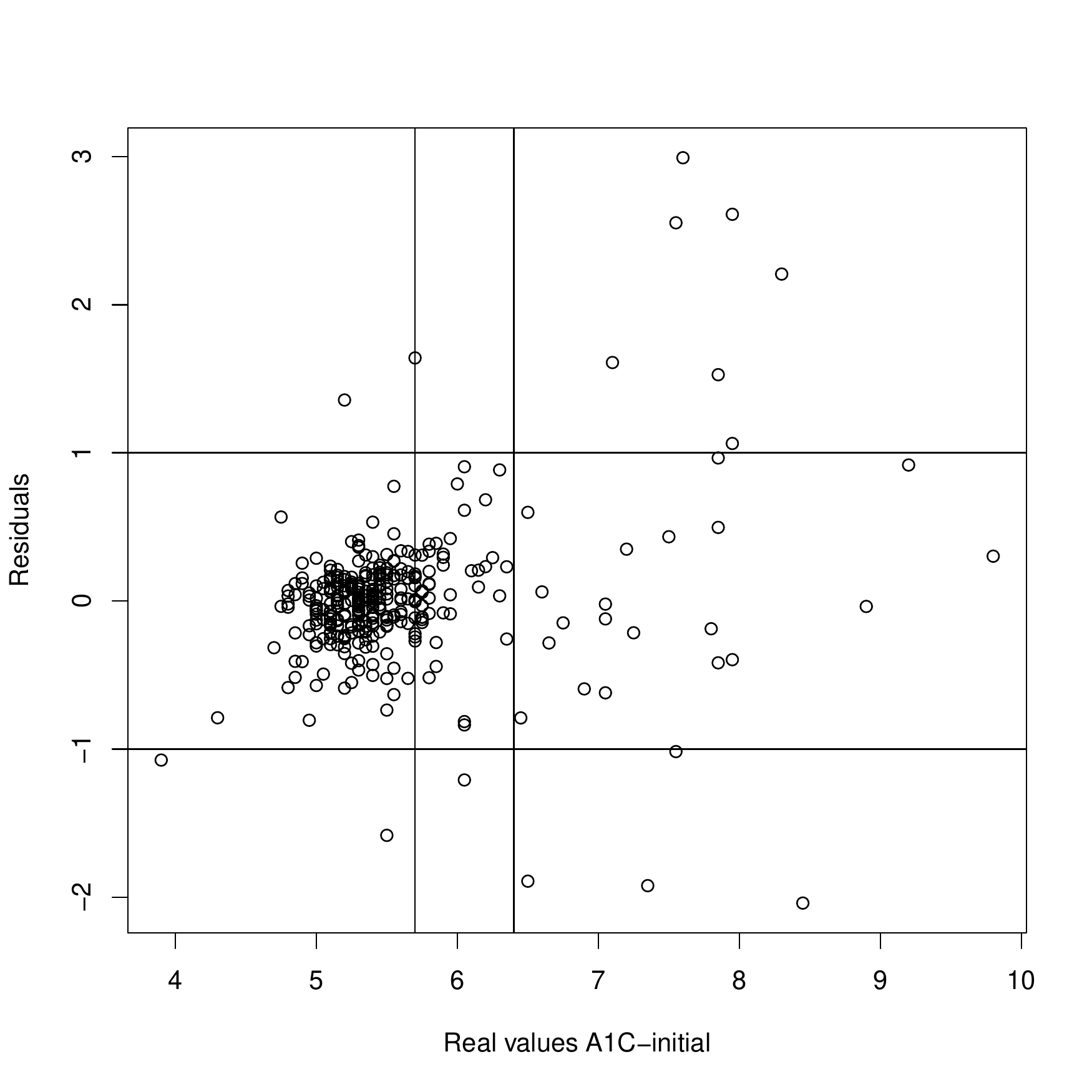}
	\caption{Residuals vs. A1C-initial for the model that includes glucodensity as a covariate.}
	\label{fig:grafico2pred}
\end{figure}

Below, we consider that a patient has a considerable uncertainty in their A1C prediction if the length interval is greater than $1$.  In this case, we can characterize clinical features that allow to assign a patient high-low variability groups based on  future glucose values uncertainty. In particular, following Figure \ref{fig:arbol}, we see that if a patient was diagnosed with diabetes before  basal period, there are essential uncertainty in their glucose values in the future. In addition, the same happens if the patients have an elevated HOMA-IR, overweight and advanced age.

\begin{figure}[ht!]
	\centering
	\includegraphics[width=0.9\linewidth]{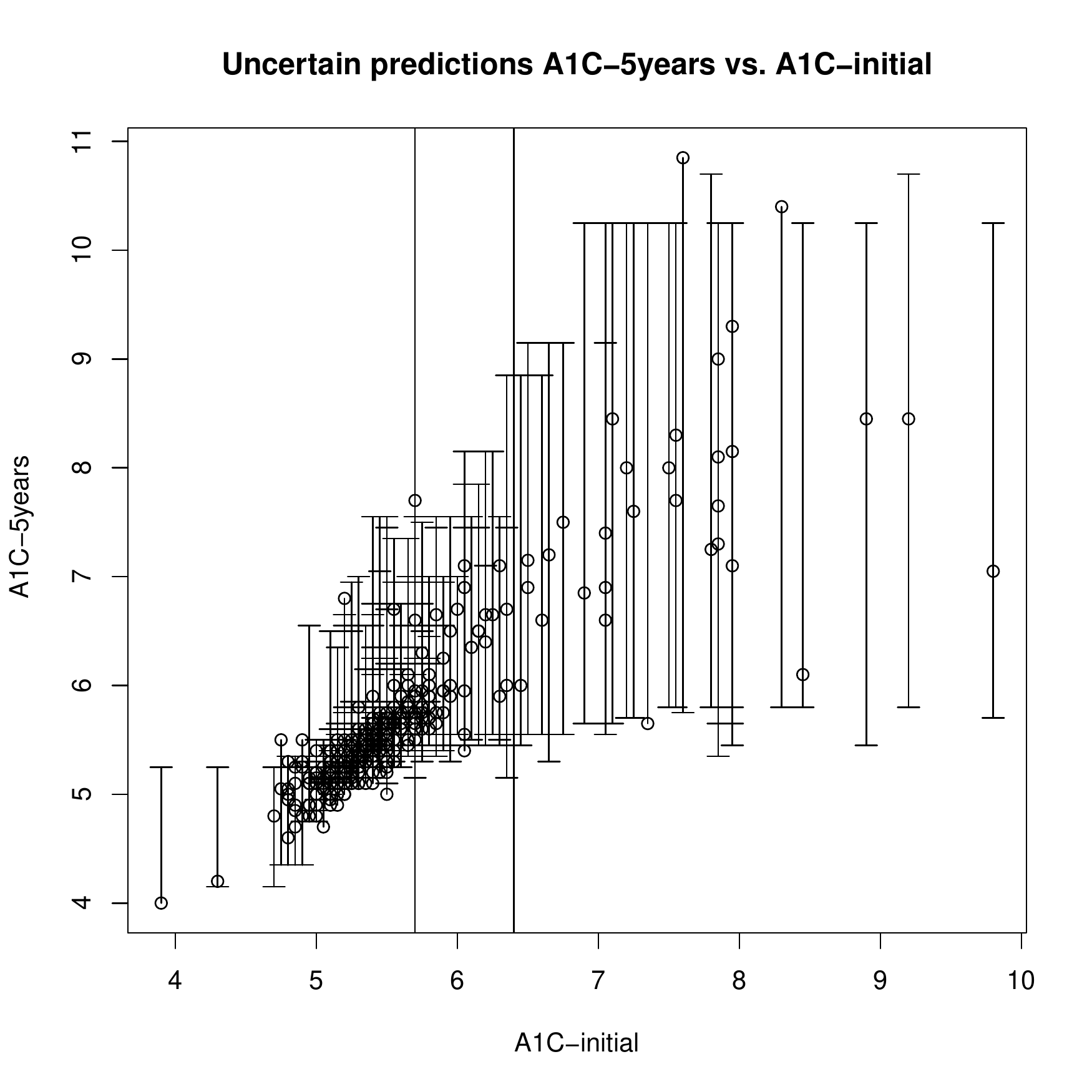}
		\caption{Conformal inference intervals of kernel ridge regression predictions for each observed response of A1C in the AEGIS database 90$\%$ nominal level coverage .}
	\label{fig:grafico3pred}
\end{figure}

\newpage

\begin{figure}[ht!]
	\centering
\includegraphics[width=0.8\linewidth]{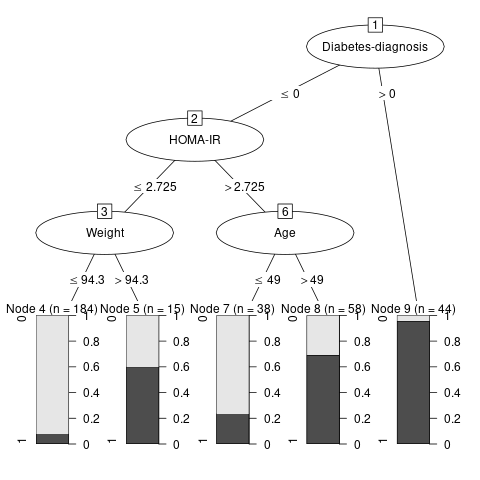}
		\caption{Collection of clinical decision rules that allow us to identify whether a patient is going to have a lot of uncertainty in their A1C predictions or not.  We establish that there is much uncertainty in the patient's glucose changes if the prediction's conformal interval is longer than 1.}
		\label{fig:arbol}
\end{figure}

\section{Discussion}

The incidence and proliferation of diabetes is one of the most critical public health problems in the world \cite{ginter2013type}. With the purpose to gain new clinical knowledge in this field and support medical decision making, the relationship between patient basal characteristics at the start of the AEGIS study and five-year A1C values has been modeled. First, we have identified several biomarkers associated with five-year glucose variations. Second, we have analyzed two nonlinear regression models’ predictive capacity to forecast A1C, showing the advantages of CGM technology for this predictive task. Finally, we have identified clinical characteristics of patients who produced unacceptable fittings from a clinical point of view. In particular, we identify some patient phenotypes that need to be tracked with more attention. We easily described them using routine biomarkers of standard clinical practice.

In order to improve clinical decisions, such as the design of optimal dynamic interventions \cite{tsiatis2019dynamic}, it is essential to design tools that quantify the future state of a patient’s homeostasis based on his or her current glycemic profile and other clinical variables. This problem is also fundamental to the identification of patients at risk of developing diabetes or in the detection of risk situations or other complications such as retinopathy. Here we have seen that we can predict this relationship in an acceptable way five years ahead, measuring A1C with a sample that includes diabetic and non-diabetic patients. However, in some patients, the discrepancies are significant. Changes in a patient’s body composition, pharmacological treatments, lifestyle such as physical activity patterns, diet, or disease development over these five years could explain partially these changes from the biological point of view.

Our phenotypes show that both insulin resistance- the previous diabetes diagnosis- and overweight explain if the performed predictions show a considerable level of uncertainty according to our model. From a practical point of view, this means that there is a more significant variability in glucose changes five years ahead in these patients, and as a result, their future glycemic status is uncertain. It is advisable to perform more routine follow-ups of this group of patients; and the performed interventions should have a personalized focus based on their dynamic evolution against treatments.

The $R^2$ value obtained predicting A1C via leave-one cross-validation with the introduction of CGM information is similar to the one reported by other authors, even though some of them predict these biomarkers in the short term \cite{gaynanova2020modeling, zaitcev2020deep}. Introducing continuous glucose monitoring information through the concept of glucodensity \cite{matabuena2020glucodensities} may provide extra information on glucose fluctuations and provide more accuracy in the prediction of A1C. However, other studies do not use a random sample; the patients are in standardized conditions and analyze different target populations; consequently, direct and accurate comparison between clinical findings and studies is not an easy issue.

Using statistical dependence measures with normoglycemic patients to test the marginal association of biomarkers with A1C demonstrates that we must use multivariate models to capture the complexity of long-term glucose changes. In this sense, we can naturally introduce several sources of information into models simultaneously with Kernel ridge regression or other RKHS techniques as the HSIC dependence measure. These data sources can include glucose profiles through glucodensities or other sources of information that have not been considered in the present work, such as omics data, and may have a significant impact on the glucose changes \cite{gou2020interpretable}. We can detect higher-order interactions between model variables as well.

From a methodological perspective, new extensions of the proposed models arise naturally as doubly robust approaches and longitudinal models that dynamically allow to introduce patient condition changes in body mass index or other relevant variables and update model predictions in real-time. At the same time, in order to improve models performance, as a future work topic it is exciting to explore the possibility of fitting some model parameters with powerful techniques of machine learning that combine models such as Super Learner \cite{van2007super}, create specific semi-parametric models for this domain \cite{tsiatis2007semiparametric}, or use other distances/kernels in our RKHS framework.

This paper predicts how much the primary variable in diabetes diagnosis and control changes. However, glucose metabolism is very complex, and other decision criteria derived from CGM data may be used in standard clinical routines that capture other aspects of glucose metabolism that go beyond glucose mean \cite{hirsch2010beyond,beyond2018need}. In this sense predicting changes in glucodensities five years ahead using the baseline data of the AEGIS study would be exciting, and it can provide a clear picture of glucose evolution values in time at the distributional level.

\section{Conclusion and practical implications}

This work proposed a new framework of data analysis based on RKHS learning when some response entries are missing. This situation is commonplace in medical studies when testing patients' evolution in different periods. Our new tools allow testing statistical independence, selecting variables, predicting, and making inferences about the predictions in the context of missing responses. As a relevant example of application, we have illustrated the usefulness of these methods for predicting glucose progression five years ahead, including the novel introduction of continuous glucose monitoring as a predictor.  The results show that predicting glucose homeostasis' evolution with continuous monitoring can provide more information than widely used classic diabetes biomarkers. Our predictive model can support clinical medical decisions with the identification of patients at risk for developing diabetes or complications in some groups of patients where model uncertainty is low. Finally, we have characterized the phenotype of patients with many discrepancies between real and predicted values using easily-measured variables. We must plan more personalized follow-ups for these patients taking into account dynamic changes in patients' conditions.

\section*{Acknowledgment}

This work has received financial support from Instituto de Salud Carlos III (ISCIII), Grant/Award Number: PI16/01395; Ministry of Economy and Competitiveness (SPAIN)  European Regional Development Fund (FEDER);  the Axencia Galega de Innovación, Consellería de Economía, Emprego e Industria, Xunta de Galicia, Spain, Grant/Award Number: GPC IN607B 2018/01; Spanish Ministry of Science, Innovation and Universities (grant RTI2018-099646-B-I00),  Galician Ministry of Education, University and Professional Training (grant 2019-2022 ED431G-2019/04).

\bibliographystyle{apalike}

\bibliography{bibliografia.bib}

\section*{Appendix: Bootstrap consistency}\label{sec:apendice}

\begin{lemma}\label{lemma:1}
 Following the notation established in Section \ref{sec:resumenresults} and \ref{sec:independence}, let us suppose that $E(K_X(X,X')^2)<\infty, \phantom{s}$ $E(K_Y(Y,Y')^2)<\infty, \phantom {s}$ $\pi(\cdot)= \mathbb{P}(R=1|X=\cdot)$ is a known two times differentiable function that verify $\pi(x)>0$  for all $x\in V$ and that all the probability weights associated have a regularly varying tail according to \cite{ma2010robust}. Under these conditions, the empirical and bootstrap statistics defined in Section \ref{sec:independence} are still consistent for detecting all second-order finite-moment alternatives with the Hilbert-Schmidt independence measure.
\end{lemma}

\begin{proof}
	
As in the Section $\ref{sec:resumenresults}$, we asume that we observe  $\{(X_i, Y_i, R_i)\}_{i=1}^{n}$  an independent random sample of a random vector $(X, Y, R)$ taking values in $V\times \mathbb{R}\times \{0,1\}$, where for simplicity in the using of empirical processes tools we assume that $V$ is the real line, although the proof remains valid in the general case.
	
Firstly note that under standard conditions, the empirical process of the distribution function $P_{X}: \sqrt{n}(P_{n,X}- P_{X})$ converges by the central functional limit theorem to a Gaussian process of mean zero which is a Brownian bridge and which we denote by $G_{X}$ \cite{van1996weak}.  

We define the class of functions $F_{(X,Y)}=\{f_{t'}:t'\in V \times  \mathbb{R}\}$, where   $f_{t'=(t_x,t_y) }=((y,x),r)= (1\{t_y \leq y,  t_x\leq  x,  r=1\}, 1\{t_y \geq  y,  t_x\geq  x\}, 1\{t_y \leq y, r=1\},  1\{t_x\leq   x, r=1\}, 1\{t_y \geq y \},  1\{t_x\geq  x\})$. If we consider  the empirical and population measures of $(Y,R)$ and $(X,Y,R)$  denoted as $P_{Y}$, $P_{X,Y}$ $P_{n,Y}$, $P_{n,X,Y}$, we can see this probability measures as random functions in the space $A(F_{(X,Y)})= \{Q\in M_{(X,Y)}: \sup\{\int f dQ  <\infty :f\in F_{(X,Y)}\}<\infty\}$, where $M_{X, Y}$ is the space of all probability bidimesional measures with moment of second order in $\mathbb{R}$. We can see that under this conditions, $F_{Y}$ and $F_{X,Y}$ are a Vapnik–Chervonenkis class and hence a Donsker class \cite{van1996weak}. Then, $\sqrt{n}(P_{n,Y}- P_{Y}) \xrightarrow{D}, G_{Y}$  $ \sqrt{n}(P_{n,X,Y}- P_{X,Y}) \xrightarrow{D} G_{X,Y}$.

Adapting these arguments to empirical bootstrap processes \cite{van1996weak}, we reach analogous convergences: 

\begin{equation*}
\sqrt{n} (P^{*}_{n,X}- P_{n,X}) \xrightarrow{D} G_{X},
\end{equation*}

\begin{equation*}
\sqrt{n} (P^{*}_{n,Y}- P_{n,Y}) \xrightarrow{D} G_{Y},
\end{equation*}

\begin{equation*}
\sqrt{n} (P^{*}_{n,X,Y}- P_{n,X,Y}) \xrightarrow{D} G_{X,Y}.
\end{equation*}

Our statistic aim  (see  Equation \ref{eqn:bootstrap1}) can be seen as a mapping  of the empirical distribution via the IPW estimator. As the missing data mechanism given by the function  $\pi(\cdot)$ (differentiable function) is known,  we can apply delta functional method to empirical process and empirical bootstrap process and obtain the Gaussian convergence of estimators modified via IPW estimator \cite{van1996weak}. We denote as follows the dependence of this estimator via IPW weights as follows $P_{n,X,Y,w}$, $P^{*}_{n,X,Y,w}$ $\ldots$ and we refer to the limit empirical process $\sqrt{n} (P^{*}_{n,X,Y,w}- P_{X,Y}) \xrightarrow{D} G_{X,Y, w}$, $\dots$.    

Our statistic  (see  Equation \ref{eqn:bootstrap1}) depends on the differences $\phi_{X}(\cdot) \phi_{Y}(\cdot)-\hat{\phi}_{X}(\cdot) \hat{\phi}_{Y}(\cdot)$, $\ldots$ which are mappings of the previous empirical processes into our selected RKHS. Suppose that these mappings defined by kernel mean embeddings are Hadamard differentiable \cite{van1996weak}, for example, in the case of Gaussian kernel. If so, we can apply the functional delta method twice, first to the kernel mean embedding and second to the dot product of RKHS to derive the empirical bootstrap process' limit distribution. The explicit limit distribution depends on the missing data mechanism; however, we can see that the limit for statistics, in general, is an infinite linear combination of Chi-square distributions \cite{korolyuk2013theory}. If the mapping were not differentiable, we should use quasi-Hadamard differentiable arguments \cite{beutner2012deriving}; or other techniques in general bootstrap measures \cite{gine1990bootstrapping}.

We can prove the test consistency of statistics as follow. First, we must establish that statistics converge to population quantity, that is $\widehat{HSIC}(\hat{P}_{X, Y},\hat{P}_{X}\hat{P}_{Y})  \xrightarrow{p} HSIC(P_{X, Y},P_{X}P_{Y})$. For this, we only need to apply the Continuous Mapping Theorem \cite{van1996weak} and take into account that the variance of the empirical kernel mean embeddings functions (mean function in appropriate RKHS space)  converges to zero as $n$ increases to infinity. To see this, we can use the fact that  $P_{n, X,Y,w}\to P_{X,Y}$ in supreme norm as class of previosly functions $F_{(X,Y)}$ are Glivenko-Cantelli class because of this is  Vapnik–Chervonenkis class \cite{van1996weak}. Then, apply twice the Continuous Mapping Theorem, first on kernel mean embedding, second on dot product in selected RKHS (that it is continuous for RKHS belong to Hilbert Space), and guarantee the final convergence with the arithmetic properties on convergence in probability.  Second, we must consider the non-negativeness of the empirical and population Hilbert Schmidt dependence measure, together with the fact that at the population level, the dependence measure is zero by definition if and only if $X$ and $Y$ are independent (null hypothesis),  diverging stochastically to infinity otherwise.

We know that the introduced statistical bootstrap imitates the limit distribution under the null hypothesis; therefore, the bootstrap's consistency is also guaranteed.  \cite{janssen2003bootstrap}.

Note that our statistic can be written as a V-statistic. In this case, the simple bootstrap is not consistent in general \cite{arcones1992bootstrap}, and it is necessary to resort ourselves to other strategies as subsampling \cite{politis1994large} or a centered statistic with respect to the mean. In fact, we are implicitly doing it with the bootstrap residues that we construct.

Finally, we must note that this paper is the first paper to apply and provide theoretical guarantees of Efron Boostrap with Hilbert-Schmidt criterium or distance correlation with missing data to the best of our knowledge. However, the arguments remain valid in the complete data case, and in this framework, we miss a paper that applies this test calibration strategy and uses empirical process theory. 

\end{proof}

In practice, we do not know the function $\pi(\cdot)$, and we must estimate it. This procedure can be done using simple techniques such as logistic regression or more complex ensemble approaches as Super Learner \cite{van2007super}.			   

In general, providing a similar proof to that in Lemma \ref{lemma:1} without introducing assumptions about the functional class of the missing data mechanism is impossible. In literature concerning IPW estimators, research often restricts its hypothesis to allowing the function $\pi_{\theta \in \Theta}(\cdot)$  to depend only on a finite set of parameters $\theta\in \Theta$  of a finite-dimensional space,  that is, under the framework of parametric models. A more general approach included the theory of M-estimator or the use that the missing data mechanics model parameter admits Bahadur expansion \cite{he1996general}.

In these cases, the proof scheme is similar to that used to prove the likelihood estimator's asymptotic properties via Taylor expansion.

In this paper, we restrict our attention to the simply case that $\quad R  \perp Y| X$ and $\pi_{\beta^*\in \Theta}(x)= \frac{1}{1+e^{-\beta_0 -\beta^{T} x}}$; what is to say, $R|X$ takes values as in a generalized lineal model according to logistic regression with $\beta^* \in \Theta=\{(\beta_0,\beta)\in \mathbb{R}\times \mathbb{R}^{p}\}$. In this case, we have:

\begin{theorem}\label{theorem:1}
	Following the notation established in Section \ref{sec:resumenresults} and \ref{sec:independence}, let us suppose that $E(K_X(X,X')^2)<\infty, \phantom{s}$ $E(K_Y(Y,Y')^2)<\infty, \phantom {s}$ $\pi(x)= \mathbb{P}(R=1|X=x)=\frac{1}{1+e^{-\beta_0 -\beta^{T}  x}}>0$ and that sufficient conditions of Gaussian asymptotic maximal likelihood hold $x\in V$. Moreover, we assume that all the probability weights associated have a regularly varying tail according to \cite{ma2010robust}. Under these conditions, the empirical and bootstrap statistics defined in Section \ref{sec:independence} are still consistent for detecting all second-order finite-moment alternatives with the Hilbert-Schmidt independence measure.
\end{theorem}

\begin{proof}
	We provide a scheme of the proof, which assumes stronger conditions than those in Lemma \ref{lemma:1} through the same arguments in its final part.

A simple proof consists on noticing Tightness of IPW process since we have  finite-dimensional convergence to Gaussian limit according to \cite{ma2010robust} under assumptions above.  Following \cite{boistard2017functional} we deduce the tight character. Then we can repeat the arguments of the functional delta theorem and continuous mapping theorem in the Lemma with Hilbert-Schmidt measure. Alternatively, one could exploit the fact that our empirical IPW statistics that depend on logistic regression can be seen as Z-estimator (see \cite{kosorok2007introduction} for more details about the topic). Finally, we can map the empirical estimator to an statistic of contrast to obtain the desired results.

\end{proof}



\end{document}